\documentclass[lettersize,journal]{article}
\usepackage{amsmath,amsfonts}
\usepackage{algorithmic}
\usepackage{algorithm}
\usepackage{array}
\usepackage[caption=false,font=normalsize,labelfont=sf,textfont=sf]{subfig}
\usepackage{textcomp}
\usepackage{multirow}
\usepackage{stfloats}
\usepackage{url}
\usepackage{verbatim}
\usepackage{graphicx}
\usepackage{cite}
\usepackage[table,xcdraw]{xcolor}
\usepackage{xcolor}
\newtheorem{theorem}{Theorem}[section]

\newtheorem{lemma}[theorem]{Lemma}

\newtheorem{example}[theorem]{Example}

\newcommand{\R}{\mathbb R}
\newcommand{\Ll}{\mathcal{L}}

\newcommand{\N}{\mathcal{N}}
\newcommand{\Hh}{\mathcal{H}}

\newcommand{\ro}[1]{{\color{blue} #1}}
\newcommand{\ma}[1]{{\color{purple} #1}}

\newcommand{\qed}{\hfill$\Box$}

\DeclareMathOperator{\Sd}{Sd}
\DeclareMathOperator{\bary}{bar}

\DeclareMathOperator{\softmax}{softmax}

\begin{document}

\title{SIMAP: A simplicial-map layer for neural networks}

\author{
Rocio Gonzalez-Diaz\thanks{
R. Gonzalez-Diaz is with the Department of Applied Math I, School of Engineering, University of Sevilla, Seville, Spain.}, 
Miguel A. Gutiérrez-Naranjo\thanks{M.A. Gutiérrez-Naranjo is with the Department of Computer Science and Artificial Intelligence, School of Engineering, University of Sevilla, Seville, Spain.}, 
\\Eduardo Paluzo-Hidalgo\thanks{E. Paluzo-Hidalgo is with the Department of Quantitative Methods, Universidad Loyola Andalucía, Campus Sevilla, Dos Hermanas, Seville, Spain.\\
Authors are listed alphabetically.}
}

\markboth{Journal of IEEE Transactions on Neural Networks and Learning Systems}%
{R. Gonzalez-Diaz, M. A. Gutiérrez-Naranjo, E. Paluzo-Hidalgo: 
SIMAP: A simplicial-map layer for neural networks}

\maketitle

\begin{abstract}
In this paper, we present SIMAP, a novel layer integrated into deep learning models, aimed at enhancing 
the interpretability of the output.
The SIMAP layer is an enhanced version of Simplicial-Map Neural Networks (SMNNs), an 
explainable neural network based on support sets and simplicial maps (functions 
 used in topology to transform shapes while preserving their structural connectivity).
The novelty of the methodology proposed in this paper is two-fold: 
Firstly, SIMAP layers work in combination with other deep learning architectures as an interpretable layer substituting classic dense final layers. 
Secondly, unlike SMNNs,
the support set is based on a fixed maximal simplex, the barycentric subdivision being efficiently computed with a matrix-based multiplication algorithm.
\end{abstract}

\section{Introduction}
In the current landscape of Artificial Intelligence, Deep Learning (DL) models have evolved to possess enormous architectures characterized by a multitude of parameters. These models, equipped with substantial computational capacity, have found widespread applications in many real-world domains.

In response to the increasing complexity of DL architectures, it is imperative to develop methods that improve the interpretability and explainability of the models.
In this way, addressing the opacity inherent in current DL architectures is not only a technical challenge but also a key step toward fostering trust and understanding in the deployment of artificial intelligence systems. This paradigm shift emphasizes the importance of developing models that not only deliver accurate predictions but also provide clear and understandable reasoning for their decisions.

Interpretable layers can significantly contribute to reliable AI systems by improving understanding to better assess whether the AI system is functioning as intended.
This term is typically used to describe a layer in a neural network where the operations and transformations it performs are understandable to humans. In this way, interpretability focuses on the transparency of the process.

In the literature, several approaches can be found to this aim.
One of the first attempts to develop interpretable layers is \cite{Zhang_2018_CVPR}, where each filter in the convolutional layer is activated by a certain object part of a certain category and remains inactivated on images of other categories. A recent survey that gives an explanation of a wide range of interpretable neural networks is
\cite{Liu2023InterpretableNN}.
In that paper, the authors divide interpretable neural network methods into two primary categories: model decomposition neural networks that involve the fusion of a conventional model-based method's interpretability with the neural network's learning capability for a clearer understanding of the neural network's operations; and semantic interpretable neural networks that derive their interpretability from a human-centric perspective, using visualization and semantic understanding.

According to this taxonomy, the approach developed in this paper can be considered as a model-decomposition neural network since we specifically construct a neural network layer, the so-called {\it SIMAP layer},  based on a simplicial map, which is a concept of algebraic topology that links two triangulated spaces in such a way that incidence relations are preserved.
Furthermore, SIMAP layers can be considered as an evolution of Simplicial Map Neural Networks (SMNNs), first defined in \cite{PaluzoHidalgo2020TwoHiddenLayerFF}.
SMNNs are a type of feedforward neural network that specifically uses
simplicial maps. This makes them useful as interpretable machine learning models.
As already demonstrated in \cite{PaluzoHidalgo2023ExplainabilitySM}, SMNNs are explainable neural network models, as they can rationalize the outputs. Specifically, an SMNN provides a justification based on similarities and dissimilarities of the data instance
to be explained with the instances of the training dataset that correspond
to the vertices of the simplex 
that contains it, after considering the dataset as a 
point cloud embedded in a metric space. The drawbacks of the SMNN approach are mainly twofold. First, the arbitrariness of the selection of a small subset of the training dataset needed to compute the Delaunay triangulation and, second, the computation of the Delaunay triangulation itself. Let us recall that the computational complexity of the Delaunay triangulation increases significantly in higher dimensions. 
In particular, for higher dimensions, the construction of a Delaunay complex becomes challenging, even for datasets containing more than a few hundred points.
This phenomenon is associated with the curse of dimensionality (see \cite{Chang2020DelaunayTriangulation}).

The SIMAP layers overcome the mentioned drawbacks by first training them using the barycentric coordinates of the input data with respect to a specific simplex surrounding the dataset. 
In doing so, we avoid the need to extract a small subset of the input data set and compute the Delaunay triangulation. In addition, we demonstrate that the capacity of a SIMAP layer increases with successive barycentric subdivisions of the simplex.
We also prove that the barycentric coordinates of the input data after the subdivision are obtained just by matrix multiplications. 
In this way, the vertices of the simplex that contain an input point are no longer part of the training set, improving, at the same time, the interpretability of the model, as the entire process becomes transparent and easily understandable to humans. 

Another approach that establishes a topology-based layer is
\cite{pmlr-v108-gabrielsson20a} where the authors introduce a topology layer that computes persistent homology (a tool that captures how topological features change
over an increasing sequence of complexes).
In contrast to the SIMAP layer, the layer defined in \cite{pmlr-v108-gabrielsson20a}
needs to have on hand an ordering of the points of the data set to induce filtration, i.e. an increasing sequence of triangulations.
A different approach is the recent field of topological deep learning (TDL), where the key idea is that input data can have a rich structure, including graphs, hypergraphs or cell complexes \cite{hajij}. 
As explained in \cite[page 60]{hajij}, ``the architecture of a TDL network can make it difficult to interpret the learnt representations and understand how the network is making predictions".
On the contrary, so far, SIMAP layers only support datasets embedded in a Euclidean space $\R^n$ 
but, as we will see later, they are clearly interpretable.

This paper is organized as follows. 
In Section \ref{sec:background}, we introduce the background and some basic definitions. 
Section \ref{subsec:tetra} is the paper's main section and is devoted to introducing the novel simplicial map layer based on barycentric coordinates and barycentric subdivisions. 
In Section \ref{sec:exper}, some experiments are performed to evaluate the effectiveness and reliability of the proposed method. 
We end the paper with a section devoted to conclusions and future work.

\section{Background}\label{sec:background}

In the following, we recall concepts such as simplicial complexes, simplicial maps, barycentric coordinates, and barycentric subdivisions, which will be later combined to improve the interpretability of deep neural network models.

\subsection{Combinatorial topology}

Let $V$ be a finite nonempty set whose elements are called vertices. A simplicial complex $K$ with vertex set $V$, is a collection  of nonempty subsets of $V$ satisfying that:
\begin{itemize}
    \item for each vertex $v$ in $V$, the set $\{v\}$ is in $K$.
    \item If $\tau$ is in $K$ and $\sigma\subset \tau$, then $\sigma$ is in $K$.
\end{itemize}
Every non-empty subset in $K$ is called a simplex; its dimension is given by its cardinality minus one. 
Given two simplices $\sigma$ and $\tau$ in $K$, we say that $\sigma$ is a face of $\tau$ whenever $\sigma\subset \tau$.
The simplices in $K$ that are not the face of other simplices of $K$  are called maximal simplices.

Let us consider a simplicial complex $K$ with a set of vertices $V\subset\R^n$ such that any simplex in $K$ is a set of (affinely independent) points. Let
$\sigma=\{v^0,v^1,\dots,v^d\}$ be a $d$-simplex of $K$. Then, the geometric realization of $\sigma$, denoted by $|\sigma|$, is defined by the convex hull of the vertices of $\sigma$:
\begin{equation*}
\begin{split}
|\sigma|=\Big\{\mbox{$x=\sum_{j=0}^d b_j(x) v^j$} \ \mid \   \mbox{$\sum_{j=0}^d b_j(x) =1 $ and \hspace{1.2cm}  }
\\ 
b_j(x)\ge 0 \ \forall j\in\{0,1,\dots,d\} 
\Big\}\,.    
\end{split}
\end{equation*}
The vector  $b(x)=(b_0(x),b_1(x),\dots,b_d(x))$
is called the barycentric coordinates of $x$ with respect to $\sigma$. 
Furthermore, the geometric realization or polytope of $K$,
denoted by $|K|$, is the union of the geometric realization of the simplices of $K$. 

The barycentric subdivision of a simplicial complex $K$,
denoted by $\Sd K$, is a simplicial complex whose vertex set is the set of the barycenters of all the simplices of $K$ (see Example~\ref{example:barycentric}).
Recall that the barycenter of a simplex
\ma{$\sigma=\{v^0,\dots,v^d\}$} is the 
point
$$\mbox{$\bary \sigma=\frac{1}{d+1}\sum_{i=0}^d v^i$}.$$
Besides, for each dimension $i\ge 0$, $\mu$ is an $i$-simplex of $\Sd K$ if and only if its vertices can be written as an ordered set $\{ w^0,w^1,\dots,w^i \}$ such that $w^k=\bary \sigma_k$ for $k\in\{0,1,\dots,i\}$ and  $$\sigma_0\subset \sigma_1\subset\cdots \subset\sigma_i
\in K \ .$$
 
The barycentric subdivision can be iterated as many times as needed.
If $K$ is subdivided $n$ times, then the corresponding simplicial complex is denoted by $\Sd^n K$.

Let us consider now a simplicial complex $K$ with vertex set $V=\{v^1,v^2,\dots,v^{\beta}\}\subset\R^n$. 
Consider also a simplex $\sigma=
\{
v^i 
\}_{i\in I}
$ of $K$, with $I\subset \{1,2,\dots,\beta\}$, 
and a point $x\in |\sigma|\subset\R^n$.
Let $b(x)$ be the barycentric coordinates of $x$ with respect to $\sigma$.
Then, the barycentric coordinates of $x$ with respect to the simplicial complex $K$ are given
by the vector  
$$\xi(x)=(\xi_1(x),\xi_2(x),\dots, \xi_\beta(x))
$$
where 
$\xi_i=0$ if $i  \not\in 
I$ and $(\xi_i(x))_{i\in I}=b(x)
$.

\begin{figure}[ht]
    \centering
\includegraphics[width=0.35\textwidth]{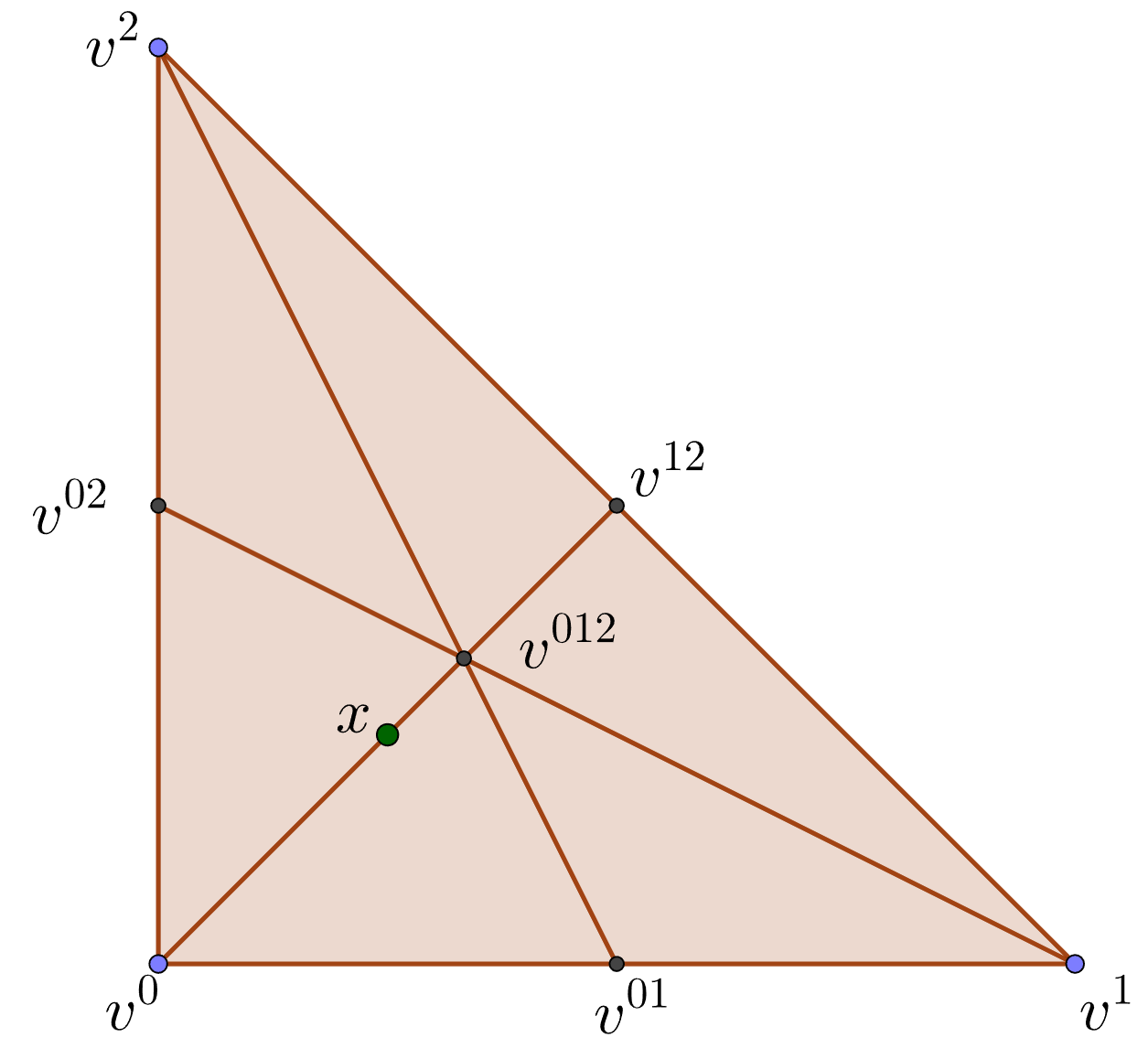}  
    \caption{The barycentric subdivision of the simplicial complex described in    Example~\ref{example:barycentric}.}
    \label{fig:example1}
\end{figure}

\begin{example}\label{example:barycentric}
Let $K$ be a simplicial complex with a set of vertices $\big\{v^0=(0,0),v^1=(2,0),v^2=(0,2)\big\}$. Then, $K=\big\{\{v^0\},\{v^1\},\{v^2\},\{v^0,v^1\},\{ v^1,v^2\},$ $\{ v^0,v^1,v^2\}\big\}$ and $\Sd K$ is composed of the following simplices (see Figure~\ref{fig:example1}):
\begin{itemize}
    \item Vertices: $\{ v^0,
 v^1,
 v^2,
 v^{01},
 v^{02},
 v^{12},
 v^{012}\}$;
\item $1$-simplices: $\big\{
\{ v^0,v^{01}\},
\{ v^0,v^{02}\},
\{ v^0,v^{012}\},\\
\hspace*{2.15cm}\{ v^1,v^{01}\},
\{ v^1,v^{12}\},
\{ v^1,v^{012}\},\\
\hspace*{2.15cm}
\{ v^2,v^{02}\},
\{ v^2,v^{12}\},
\{ v^2,v^{012}\},\\
\hspace*{2.15cm}
\{ v^{01},v^{012}\},
\{ v^{02},v^{012}\},
\{ v^{12},v^{012}\}
\big\}$; 
    \item $2$-simplices: $\big\{
\{ v^0,v^{01},v^{012}\},
\{ v^0,v^{02},v^{012}\},\\ \hspace*{2.15cm} \{ v^1,v^{01},v^{012}\},  \{ v^1,v^{12},v^{012}\},\\
\hspace*{2.15cm} \{ v^2,v^{02},v^{012}\},
\{ v^2,v^{12},v^{012}\}
\big\}$.
\end{itemize}
Let us now consider a point $x=(\frac{1}{2},\frac{1}{2})$. The barycentric coordinates of $x$ with respect to $K$ are $\xi_K(x)=(\frac{1}{2},\frac{1}{4},\frac{1}{4})$. The barycentric coordinates of $x$ with respect to the simplex $\{ v^0,v^{01},v^{012} \} $  are $(\frac{1}{4},0,\frac{3}{4})$. Hence, $\xi_{\Sd K}(x)=(\frac{1}{4},0,0,0,0,0,\frac{3}{4})$.
\end{example}

Given two simplicial complexes $K$ and $L$ with vertex set $V$ and $W$, respectively. A vertex map $\varphi^{(0)}: V \rightarrow W$ is 
a correspondence between vertices
such that for each simplex  $\sigma= 
\{ v^i
\}_{i\in I}$
of $K$, 
the set obtained from $\big\{ \varphi^{\scriptscriptstyle (0)}(v^{i})\big\}_{i\in I}$ after removing duplicated vertices is a simplex of $L$.
A vertex map induces a continuous function called a simplicial map defined for all $x\in |K|$ as follows:
\[\mbox{$\varphi(x)=\sum_{i\in I}b_i(x)\varphi^{(0)}(v^i)$ such that $x\in |\sigma|$
and $\sigma\in K$}.
\]


\begin{figure}[ht]
    \centering
\includegraphics[width=0.45\textwidth]{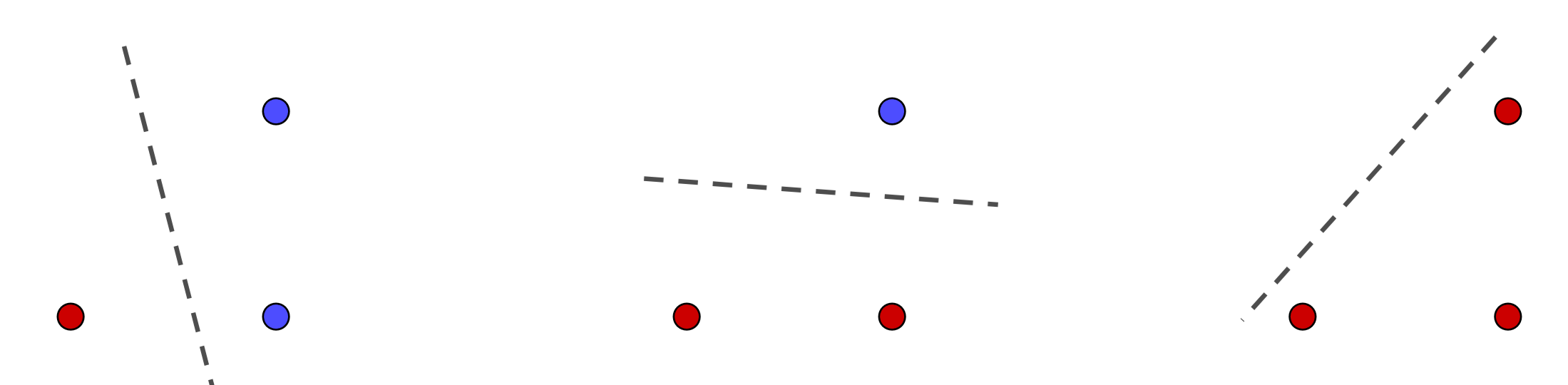}   
    \caption{All possible dichotomies of a dataset of size $3$ in $\R^2$ shattered by a line. 
    }
    \label{fig:example_VC}
\end{figure}


\subsection{Neural networks for classification}

A dataset 
\ma{$\mathcal{D}=(X,\lambda)$} of a classification problem is a pair composed of a set of points $X\subset \mathbb{R}^d$ of size $N$ and a function
$\lambda$ from $X$ to $\Lambda=\{0,1,\dots, k
\}$ where $k+1$ is the number of classes. We denote by $y_i$ the image by
$\lambda$ of $x_i\in X$ for all $i\in \{1,2,\dots,N\}$. The set of points $\{x\ | x\in X, \lambda(x)=c\}$ with $c\in\Lambda$ 
 is called the class $c$. The classification task consists of finding a function $\mathcal{N}: \mathbb{R}^d \rightarrow \Lambda$ that approximates $\lambda$.  
 
 In this paper, the function $\mathcal{N}$ is
 a neural network  defined between  
 $\mathbb{R}^d$ and $\mathbb{R}^k$ and composed of $m+1$ functions, 
$$\mathcal{N}=f_{m+1}\circ f_m\circ \cdots f_1\,,$$
 where the integer $m>0$ is the number of hidden layers and, for $i \in \{1,2,\dots, m+1\}$, the function $f_i:X_{i-1}\rightarrow X_i$ is defined as $f_i(y)=\phi_i(W^{(i)};y_i;b_i)$ where $X_0=\mathbb{R}^d$, $X_{m+1}=
 \mathbb{R}^k$, 
 $X_i\subseteq \mathbb{R}^{d_i}$, 
 with $d_0=d$ and $d_{m+1}=k$, 
 $W^{(i)}$ is a matrix $d_{i-1}\times d_i$,  
 $b_i$ is a $d_i$-dimensional vector, and $\phi_i: \R^{d_{i-1}}\to \R^{d_i}$ is a continuous function. 
 The matrices $\{W^{(i)}\}$ are  usually updated using an optimization algorithm such as stochastic gradient descent and are called the 
weight matrices
 of $\N$.

 From learning theory, we will use the concept of the Vapnik-Chervonenkis dimension (VC-dimension) to quantify the capacity of a neural network (see \cite{vapnik:264}).
 Given a point cloud $X$,  a dichotomy of  $X$ is one of the possible ways to label $X$  in the binary classification context. Then we say that a model shatters $X$  if, 
 for any possible dichotomy of $X$,  we can find the architecture of the model that correctly classifies it.
 The VC-dimension of a model is the maximum size of a dataset that can be found to be shattered by it.
 A good introduction to the VC-dimension of neural networks is \cite{sontang} where it is shown that
the  VC-dimension of perceptrons with $n+1$ entries (including the bias term)  and hyperplanes in $\R^n$ is $n+1$. For example, consider the two-dimensional case illustrated 
 in Figure~\ref{fig:example_VC}.
 All possible dichotomies of a dataset of size $3$ in $\R^2$ can be shattered by a hyperplane. However, no data set of size $4$ in $\R^2$ 
 can be shattered by a hyperplane. Therefore, the VC dimension of a hyperplane in $\R^2$ is $3$.

\section{Simplicial-map layer}\label{subsec:tetra}

In this section, we introduce the main novelty of this paper to add interpretability to deep neural networks: the simplicial-map layer, also called the SIMAP layer, 

Unlike the previous definitions of simplicial-map neural networks \cite{paluzohidalgo2023explainability}, the SIMAP layers are not based on Delaunay triangulation, but on barycentric subdivisions, so we do not need to establish the support set.
Furthermore, we prove that we do not need to explicitly calculate the successive barycentric subdivisions, since we can deduce the barycentric coordinates of a point with respect to a subdivision from its barycentric coordinates with respect to the simplicial complex that existed
before the subdivision.
Next, let us see step by step how to compute the SIMAP layers for a dataset $D=(X,\lambda)$ of a classification problem.

\subsection{Convex polytope and barycentric coordinates}
\label{subsec:poly}
First, we consider that $X\subset \R^n$ lies inside the unit hypercube $\Hh$ whose corners are the $2^n$  points with coordinates $0$ or $1$. 
If that is not the case, we can always do an affine transformation to achieve it.

Then, we compute an $n$-simplex $\sigma$ whose geometric realization (convex polytope) contains $\Hh$ and, therefore, the point cloud $X$.
The following lemma provides the coordinates of the vertices of $\sigma$.

\begin{lemma}\label{lemma_tetra}
Let $\Hh$ be the $n$-dimensional hypercube in $\mathbb{R}^n$
whose corners are the $2^n$  points with coordinates $0$ or $1$.
Then, the vertices of an $n$-simplex $\sigma$ satisfying $\Hh\subset |\sigma|$ are
the rows of the following $(n+1)\times n$ matrix
$$S=\begin{pmatrix}
0&0&\cdots&0& 0\\
n&0&\cdots&0& 0\\
0&n&\cdots& 0&0\\
\vdots&\vdots&\ddots&\vdots&\vdots\\
0&0&\cdots& n&0\\
0&0&\cdots& 0&n\\
\end{pmatrix}.$$ 
\end{lemma}
\begin{proof} 
The corner $(0,0,\dots,0)$ of $\Hh$ is a vertex of $\sigma$ (first row of $S$). The corner $c=(c_1,c_2,\dots,c_n)$ of $\Hh$ with
$c_i=1$ for $i\in I\subset \{1,2,\dots,n\}$
lies in the face of $\sigma=\{v^i\}_{i\in I\cup\{0\}}$ whose vertices $v^i$ are 
 the $(i+1)$-th rows  of $S$ for $i\in I\cup\{0\}$. Specifically, $c=\frac{1}{n}\sum_{i\in I\cup\{0\}}v^i$.
The corner $(1,1,\dots,1)$ of $\Hh$ 
is the barycenter of the face of $\sigma$ whose vertices are the $(i+1)$-th rows of $S$ for $i\in \{1,2,\dots,n\}$.
Since $|\sigma|$ is the convex hull of the vertices of $\sigma$ and the corners of $\Hh$ are in $|\sigma|$ 
then $\Hh\subset |\sigma|$. 
\qed  
\end{proof}

Now, let $\sigma$ be the simplex defined in Lemma \ref{lemma_tetra}, let $x\in |\sigma |$ and let $b(x)$ be
the barycentric coordinates of $x$ with respect to $\sigma$. Let $T=({\bf 1}\ |\ S)$ be the matrix consisting of $S$ with an additional column whose entries are all $1$.
Taking into account that, by definition of barycentric coordinates,
$b(x)  \ T =(1 \ | \ x)$, 
the barycentric coordinates of $x$ with respect to $\sigma$ can be easily computed, as the following result shows.

\begin{lemma}\label{lemma:barycentric_comp}
Let $\sigma$ be the $n$-simplex defined in Lemma~\ref{lemma_tetra} and let $x=(x_1, \dots$, $x_n) \in |\sigma|$. Then, the barycentric coordinates $b(x)=(b_0(x), \dots,b_n(x))$ of $x$ with respect to $\sigma$
can be obtained by multiplying the vector $(1,x_1,\dots.x_n)$ by the matrix $M$, i.e.,
$$ (b_0(x),\dots,b_n(x)) = (1,x_1,\dots.x_n)\,M$$
where $M$ is the $(n+1)\times (n+1)$ matrix
$$M=\begin{pmatrix}
1&0&0&\cdots&0& 0\\
-\frac{1}{n}&\frac{1}{n}&0&\cdots&0& 0\\
-\frac{1}{n}&0&\frac{1}{n}&\cdots& 0&0\\
\vdots&\vdots&&\ddots&&\vdots\\
-\frac{1}{n}&0&0&\cdots& \frac{1}{n}&0\\
-\frac{1}{n}&0&0&\cdots& 0&\frac{1}{n}\\
\end{pmatrix}$$
\end{lemma}

\begin{proof}
Using that 
$$(b_0(x),\dots,b_n(x) )\,T = (1,x_1,\dots,x_n),$$
and since $T\cdot M$ is the identity matrix then:
\begin{equation*}
\begin{split}
 (b_0(x),\dots,b_n(x) ) &=(b_0(x),\dots,b_n(x) )\,T\cdot M\\
 &=(1,x_1,\dots,x_n)\,M.   
\end{split}
\end{equation*}
\qed    
\end{proof}

\subsection{Barycentric coordinates after a barycentric subdivision}\label{subsec:subdivision}

Once we know how to compute the barycentric coordinates of a point $x\in |\sigma|$,  the next step is to find how to calculate the barycentric coordinates of $x$ with respect to a subdivision of $\sigma$.

Let us consider an $n$-simplex $\sigma=\langle  v^0,\dots,v^n\rangle$ in $\R^n$.
Then, each ordering $(i_0,\dots,i_n)$ of the indices $(0,\dots, n)$ represents one of the maximal simplices obtained by the barycentric subdivision of $\sigma$. 
This way, the set of vertices of the $n$-simplex $\sigma^{i_0\cdots i_n}$, represented by the ordering
$(i_0,\dots,i_n)$, are 
$$\{v^{i_0},v^{i_0i_1},\dots, v^{i_0\dots i_n}\}$$
where $ v^{i_0\dots i_k}=\bary \langle v^{i_0},\dots, v^{i_k}\rangle$
for all $k\in\{0,\dots,n\}$,
and $\{
v^{i_0},\dots, v^{i_k}
\} $ is a face of $\sigma$.

\begin{example}
Following Example~\ref{example:barycentric}, the simplex $\{
v^1,v^{12},v^{012} 
\}$ will be identified with
the ordering $(1,2,0)$ and will be denoted by $\sigma^{120}$.
\end{example}

\begin{lemma}\label{lem:P}
Let $\sigma^{i_0\cdots i_n}$ be a maximal simplex of the barycentric subdivision of 
 $\sigma$  represented by an ordering $(i_0,\dots,i_n)$.
 Let $x \in \R^n$ such that its barycentric coordinates with respect to $\sigma$ are $(b_0(x), \dots, b_n(x))$.
 Then, 
 if $x\in |\sigma^{i_0\dots i_n}|$,
 the barycentric coordinates
 $(b^1_0(x),\dots,b^1_{n}(x))$ of $x$ with respect to $\sigma^{i_0\dots i_n}$ satisfies that:
 $$(b^1_0(x),\dots,b^1_n(x)) = (b_{i_0}(x), \dots, b_{i_n}(x)) \, P$$
where $P$ is the $(n+1)\times (n+1)$ matrix
$$
P =
\begin{pmatrix}
 1  &  0 &  0 & \cdots  & 0 & 0 & 0   \\
 -1 &  2 &  0 & \cdots  & 0 & 0 & 0   \\
 0  & -2 &  3 & \cdots  & 0 & 0 & 0   \\
\vdots & \vdots   & \vdots & \vdots         &\vdots          & \vdots \\
 0  &  0 & 0  & \cdots  &  n-1        & 0 & 0   \\
 0  &  0 & 0  & \cdots  & - (n-1)              &  n        & 0  \\
 0  &  0 & 0  & \cdots   & 0              & -n              & n+1
\end{pmatrix}.$$
\end{lemma}

\begin{proof}
By definition, $ x = \sum_{k=0}^n b^1_k(x) v^{i_0\dots i_k}$. 
Since 
$$\mbox{$v^{i_0\dots i_k}=\bary \langle v^{i_0},\dots,v^{i_k}\rangle = \frac{1}{k+1} \sum_{j=0}^k v^{i_j}$}$$ then
\begin{equation*}
\begin{split}
x &\mbox{$= b^1_0(x)\,v^{i_0}
+b^1_1(x)\,\frac{1}{2}(v^{i_0} + v^{i_1} )
+\dots$}
\\
&\mbox{$\;\;\;+ b^1_n(x)\,\frac{1}{n+1}(v^{i_0} + \dots + v^{i_n})$}\\
&\mbox{$=\big(b^1_0(x)+\frac{1}{2}b^1_1(x)+\cdots+\frac{1}{n+1}b^1_n(x)\big)v^{i_0}+\cdots$}\\
&\mbox{$\;\;\;+\frac{1}{n+1}
 b^1_n(x) v^{i_n}$.}   
\end{split}    
\end{equation*}
As a matrix equation, $$(b^1_0(x),\dots, b^1_n(x))\, Q = (b_{i_0}(x),\dots ,b_{i_n}(x)),$$ where
$$
Q =
\begin{pmatrix}
 1  &  0 &  0 & \cdots  & 0 & 0 & 0   \\
 \frac{1}{2} & \frac{1}{2} &  0 & \cdots  & 0 & 0 & 0   \\
 \frac{1}{3}  & \frac{1}{3} & \frac{1}{3} & \cdots  & 0 & 0 & 0   \\
\vdots & \vdots   & \vdots & \vdots         &\vdots          & \vdots \\
 \frac{1}{n}  &  \frac{1}{n} & \frac{1}{n}  & \cdots  &  \frac{1}{n}             &  \frac{1}{n}        & 0  \\
 \frac{1}{n+1} & \frac{1}{n+1}  & \frac{1}{n+1}  & \cdots   & \frac{1}{n+1} &   \frac{1}{n+1} & \frac{1}{n+1}
\end{pmatrix}.$$
Therefore, $$(b^1_0(x),\dots, b^1_n(x))\, Q\cdot P= (b_{i_0}(x),\dots ,b_{i_n}(x))\,P$$ concluding the proof  since $Q\cdot P$ is the identity matrix.
\qed
\end{proof}

We end this subsection with the following result that is useful to easily detect the simplex  $\sigma^{i_0\cdots i_n}\in \Sd \sigma$
where the given point $x$ is located.

\begin{lemma}
Let $\sigma$ be the simplex defined in Lemma~\ref{lemma_tetra} and let $x$ be a point in $|\sigma|$. The point $x$ lies in $|\sigma^{i_0\cdots i_n}|$ if and only if the ordering $(i_0,\dots,i_n)$ makes that the barycentric coordinates $(b_{i_0}(x),\dots,b_{i_n}(x))$
are arranged in a non-increasing manner, that is,  $b_{i_0}(x)\geq \dots \geq b_{i_n}(x)$. 
\end{lemma}

\begin{proof}
Recall that the point $x$ lies in
$|\sigma^{i_0\cdots i_n}|$ if its barycentric coordinates $(b^1_{0}(x),\dots,b^1_{n}(x))$ are nonnegative. Applying Lemma~\ref{lem:P}, we have $b^1_{n}(x)=(n+1)b_n(x)$, which is nonnegative since $x\in|\sigma|$. Furthermore,
$b^1_{j}(x)=(j+1)(b_{j}(x)-b_{j+1}(x))$ for any $j\in\{0,1,\dots,n-1\}$, which is nonnegative if and only if $b_{j}(x)\geq b_{j+1}(x)$, concluding the proof.
\qed
\end{proof}

\subsection{Training}\label{subsec:training}

This is the core of the SIMAP layer. In the first step, without considering any subdivision,  it is simply a perceptron with $n+1$ inputs and $k$ outputs if the given dataset is a point cloud in $\R^n$ labelled with $k$ classes. After successive subdivisions, this part of the SIMAP layer can be seen as the union of $m$  perceptrons with $n+1$ inputs and $k$ outputs since the idea of this part of the SIMAP layer is that a point is classified according to the maximal simplex of the subdivision on which it lies.
 
Let $\{v^0,\dots,v^n\}$ be the vertices of the $n$-simplex $\sigma$ defined in Lemma~\ref{lemma_tetra}.
Assume that our input dataset $D=(X,\lambda)$ is a finite set of points in $|\sigma|$,  together with a set of $k+1$ labels $\Lambda=\{0,\dots,k\}$ such that  each $x\in X$ 
 is labelled with a label $\lambda(x)$
 taken from $\Lambda$. 
A one-hot encoding representation is:
 $$L=  \big\{(0,\stackrel{j}\dots,0,1,0,\stackrel{k-j}{\dots},0):\;
 j\in \{0,\dots,k\}\big\}\,,
 $$
 where the one-hot vector $(0,\stackrel{j}\dots,0,1,0,\stackrel{k-j}{\dots},0)$ encodes the $j$-th label of $\Lambda$ for $j\in \{0,\dots,k\}$ and $\ell(x)=(\ell_0(x),\dots,\ell_k(x))$ encodes the label $\lambda(x)$ for any $x\in X $.
 
Observe that if we define a vertex map $\varphi^{\scriptscriptstyle (0)}:V\to L$ mapping each vertex $v\in V$ to a point in $\R^{k+1}$. 
Then, for any $x\in |\sigma|$ 
with barycentric coordinates $b(x)=(b_0(x),\dots,b_n(x))$, we have that $\varphi(x)=\sum_{i=1}^nb_i(x)\varphi^{(0)}(v^i)$ is a probability distribution satisfying  that the $i$-th coordinate of $\varphi(x)$ indicates the probability that $x$ belongs to the class $i\in \Lambda$.
 
In general, the values of  $\varphi^{\scriptscriptstyle (0)}(v^i)$ are unknown and the procedure explained can not be applied. 
Therefore, the idea for the 
next step of our algorithm is  to  define   a multiclass perceptron denoted by $\N$ with an input layer with $n+1$ neurons that computes the probability that  $x\in|\sigma|$ is labelled with a label of $\Lambda$
using the formula:
$$\N(x)=(\N_0(x),\dots, \N_k(x))=
\softmax\big(b(x)\cdot \Omega \big)
$$
where $\Omega\in \mathcal{M}_{(n+1)\times (k+1)}$
is a weight matrix
and $b(x)\in \R^{n+1}$ are the barycentric coordinates of $x$ with respect to $\sigma$.
The training procedure has the aim of learning the values of  $\varphi^{\scriptscriptstyle (0)}(v^i)$  that minimizes the error:
\[\mbox{$\Ll(x)=-\sum_{ h=1}^k\ell_h(x)\log(\N_h(x))$}\,,\]
 for $x\in D$.
As proven in \cite{paluzohidalgo2023explainability}, 
$$\mbox{$\frac{\partial\Ll(x)}{\partial p_j^{t}}=(\N_j(x)-\ell_j(x))b_{t}(x).$}$$
Then, using, for example, gradient descent, we have to update the entries $p_j^{t}$ of $\Omega$, as follows:
\[\mbox{$p_j^{t}:=
p_j^{t}-\eta(\N_j(x)-\ell_j(x))b_{t}(x)$.}\]
Observe that the $i$-th row of $\Omega$ can be thought of as the value of
$\varphi^{(0)}(v^i)$ for $i\in \{0,\dots ,n\}$.

Since the capacity of the perceptron $\N$ is limited, the idea is to train another perceptron $\N^1$ from it, emulating a barycentric subdivision.
This new perceptron $\N^1$ is initialized as
$$\N^1(x)=\softmax\big(
\xi_{\Sd \sigma}(x)
\cdot \Omega^1 \big)$$
where $\Omega^1$
is a $(2^{n+1}-1)\times(k+1)$ matrix whose rows are in one-to-one correspondence with the value of  
$\varphi$ applied to the vertices of $\Sd\sigma$ and
$
\xi_{\Sd \sigma}(x)
$ is computed as explained in subsection~\ref{subsec:subdivision}.
Specifically, by construction, we have 
$$\xi_{\Sd \sigma}(x)
\cdot \Omega^1  =
\xi_{\Sd \sigma}(x)
\cdot Q\cdot \Omega=b(x)\cdot\Omega\,.$$   So, initially, 
$\N^1(x)=\N(x)$ for any $x\in |\sigma|$. 
Then, the weights of $\Omega^1$ are updated using gradient descent as explained above.
The process can be iterated until a given error is reached. We will denote by $VC(\mathcal{N})$ to the VC dimension of the neural network $\mathcal{N}$. We have the following result.

\begin{theorem}
Let $\{v^0,\dots,v^n\}\subset\mathbb{R}^n$ be the vertices of the $n$-simplex $\sigma$ defined in Lemma~\ref{lemma_tetra}. Let $\mathcal{N}$ be the neural network
defined from $\sigma$ and $\mathcal{N}^k$ the neural network
defined from $\Sd^k \sigma$ with $k\ge 0$ as explained above.
Then,
 $$VC(\mathcal{N}^k)= ((n+1)!)^k\cdot (n+1)\,.$$
\end{theorem}

\begin{proof}
For $k=0$, 
the neural network $\N$ defined from $\sigma$
is a perceptron with $n+1$ entries (the $n+1$ barycentric coordinates with respect to $\sigma$) and no bias  whose VC-dimension is $n+1$.
For $k>0$ 
and  $x\in |\Sd^k\sigma|$, there exists a maximal simplex $\mu$ in $\Sd^k\sigma$ such that  $x\in |\mu|$. Then, $\xi_{\Sd^k\sigma}(x)$ 
has at most
$n+1$ non-null coordinates. Since $\Sd^k\sigma$ has $((n+1)!)^k$ maximal simplices, the neural network $\N^k$ defined from $\Sd^k\sigma$
acts as $((n+1)!)^k$ independent perceptions. Then, the VC-dimension of $\N^k$ is
$((n+1)!)^k \cdot (n+1)$.
\qed  
\end{proof}

Let us remark that, by definition, there exists no simplicial map that 
maps a simplex to another simplex of higher dimension.
\begin{lemma}
Let $D=(X,\lambda)$ be a dataset with $X\subset\R^n$. Let $\sigma$ be
the $n$-simplex 
defined in Lemma~\ref{lemma_tetra}. Let $\mathcal{N}^0
$ be the neural network
defined from $\sigma$ and $\mathcal{N}^k$ the neural network
defined from $\Sd^k \sigma$ with $k>
0$ as explained above. 
Fixing $k\geq 0$, let $\mu\in \Sd^k \sigma$ 
such that $Y=\{x^0,\dots,x^{n+1}\}\subseteq X$ is a set of
points in $|\mu|$ satisfying that
no two 
points of $Y$ have
the same label. Then, $\mathcal{N}^k$ cannot correctly classify $X$.
\end{lemma}
\begin{proof}
The dimension of $\mu$ is $n+1$. 
Then, for all $x\in Y$, 
$\xi_{\Sd^k \sigma}(x)$ is possibly not null for fixed $n+1$ coordinates. Therefore, the
 output of $\mathcal{N}^k(x)=\softmax (\xi_{\Sd^k \sigma}(x)\cdot \Omega^k)$
is again possibly not null for fixed $n+1$ coordinates, so it can  cannot correctly classify $Y$ (nor $X$) since  the number of classes for $Y$ (and then, $X$) is at least $n+2$.  
\qed
\end{proof}

Summing up, given a dataset $D=(X,\lambda)$ with $X\subseteq \R^n$ and $k+1$ classes, the SIMAP layer $\N^k$ is a neural network with no hidden layer and 
acting as $\big((n+1)!\big)^k$ independent perceptrons with at most $n+1$ activated neurons in the input layer and $k+1$ neurons in the output layer. A pseudocode for SIMAP layers can be found in Algorithm~\ref{algo:SIMAP}.

\begin{figure*}[ht]
    \centering
    \includegraphics[width=0.8\textwidth]{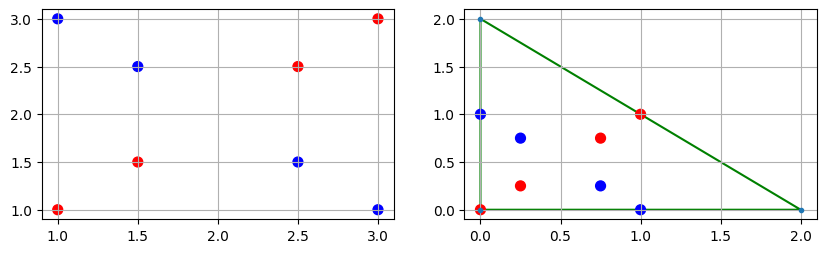}
    \caption{Dataset of Example~\ref{example:XOR}. On the left, the input data for the binary classification is shown. On the right, its translation into the simplex $\sigma$ of Lemma~\ref{lemma_tetra}.}
    \label{fig:example3}
\end{figure*}

\begin{algorithm}
\caption{Pseudocode to compute  SIMAP layers }
\label{algo:SIMAP}
 \begin{algorithmic}
 
\STATE  {\bf Input:} 
A dataset $D=(X,\lambda)$ 
with a set of labels $\Lambda$. \STATE \hspace{1cm} An integer $\ell$ (number of subdivisions).
\STATE   {\bf Output:} A SIMAP layer that generalizes $D$.
\STATE
 \IF{$X$ does not lie in the hypercube $\Hh$}
\STATE {\bf Transform} the points of $X$ so that they lie in $\Hh$
(see Subsection~\ref{subsec:poly})
\ENDIF
\STATE $k=0$
\STATE {\bf Compute} the barycentric coordinates $b(x)$ of the all points $x\in X$ (see Subsection~\ref{subsec:poly}).
\STATE {\bf Train} the perceptron 
$\N(\;)=\softmax(b(\;)\cdot \Omega)$ (as explained in Subsection~\ref{subsec:training}).
\WHILE{$k\leq \ell$}
\STATE  Compute the barycentric coordinates $\xi(x)$ with respect to $\Sd^{k}\sigma$ of the all points $x\in X$ (as explained in Subsection~\ref{subsec:subdivision}). 
\STATE Train the perceptron 
$\N^{k}(\;)=\softmax(\xi(\;)\cdot \Omega)$ (as explained in Subsection~\ref{subsec:training}).
\ENDWHILE
 \end{algorithmic}
\end{algorithm}

Below, we provide a step-by-step example following Algorithm~\ref{algo:SIMAP}
showing how to compute a SIMAP layer for a given dataset $D$.

\begin{figure}[ht]
    \centering
    \includegraphics[width=0.4\textwidth]{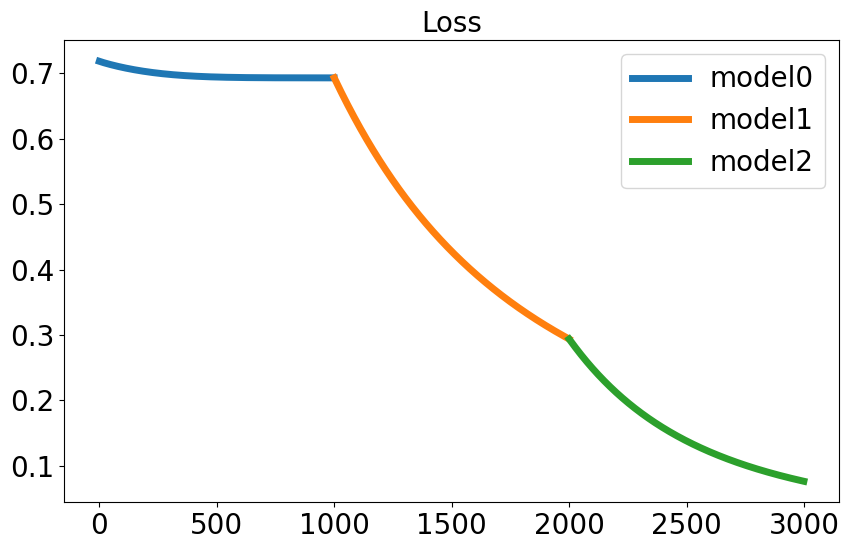}
\includegraphics[width=0.4\textwidth]{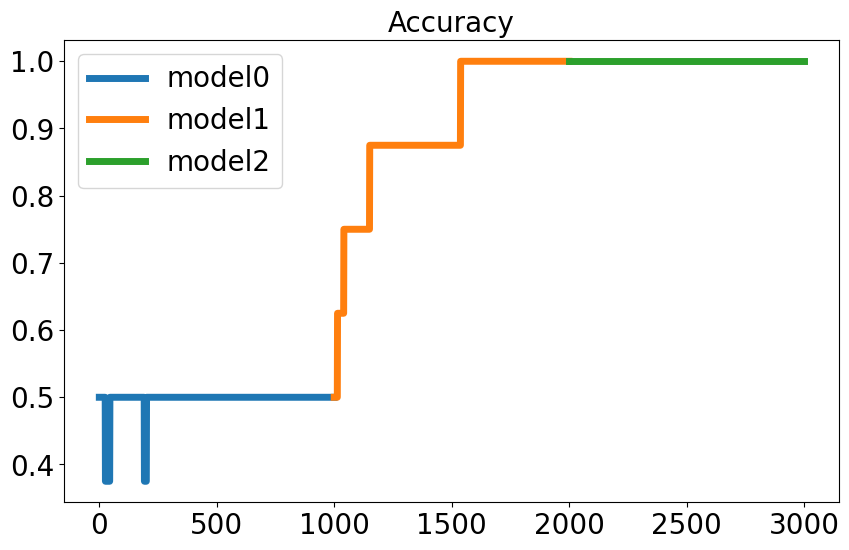}
    \caption{Training curves of Example~\ref{example:XOR}. On the top: the loss function is shown. On the bottom: the accuracy values for the different epochs are shown.}
    \label{fig:example_loss_acc}
\end{figure}

\begin{figure}[ht]
    \centering
    \includegraphics[width=0.4\textwidth]{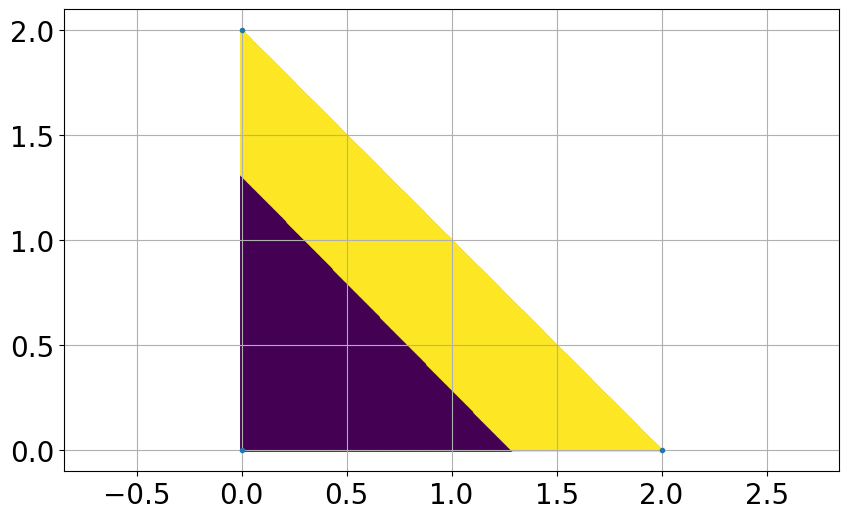}
\includegraphics[width=0.4\textwidth]{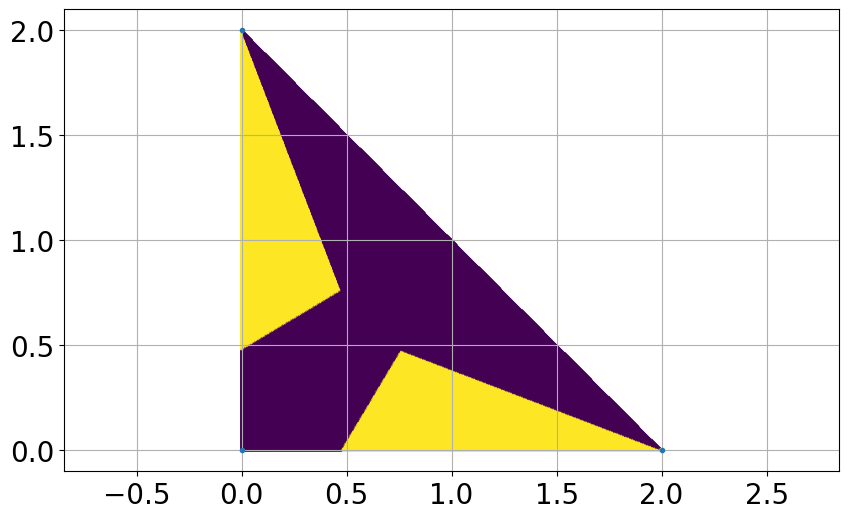}
\includegraphics[width=0.4\textwidth]{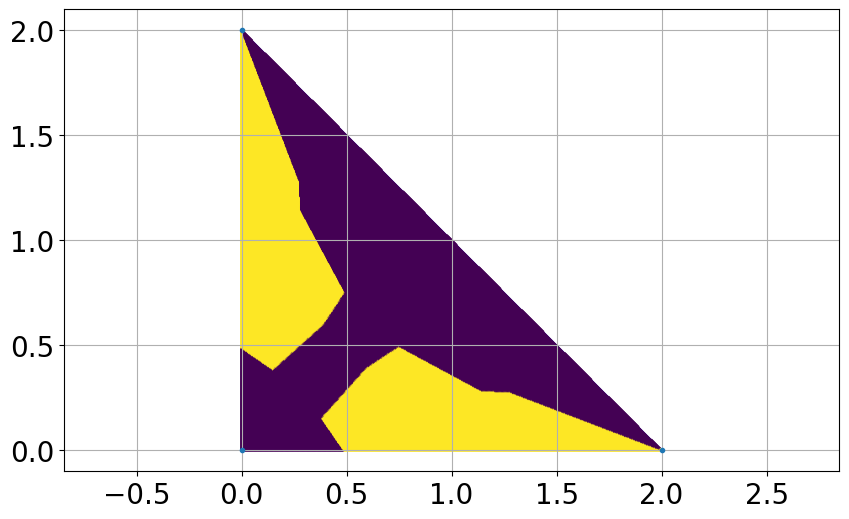}
    \caption{From top to bottom, the decision boundaries for Example~\ref{example:XOR} for the different models with respect to $\sigma$, $\Sd \sigma$, and $\Sd^2 \sigma$, respectively.}
    \label{fig:decision_boundary}
\end{figure}

\begin{example}\label{example:XOR}
Let us consider a dataset 
$D=(X,\lambda)$, with $X\subset\R^2$,
inspired by
the classic XOR-dataset (see Figure~\ref{fig:example3}). Specifically, 
$X
=\big\{v^0=(1,1), v^1=\left(\frac{3}{2},\frac{3}{2}\right),v^2=\left(\frac{5}{2},\frac{5}{2}\right), v^3= (3,3), v^4=(1,3),v^5=\left(\frac{3}{2}, \frac{5}{2}\right),v^6=\left(\frac{5}{2}, \frac{3}{2}\right), v^7=(3,1) \big\}$. The point cloud $X$
is labelled by $\lambda(v^i)=0$ if $i<4$ and $\lambda(v^i)=1$ in any other case. The steps to calculate a SIMAP layer that classifies $X$
are the following:
\begin{enumerate}
    \item Dataset preparation: the points of $X$ are centered inside the simplex $\sigma$ of Lemma~\ref{lemma_tetra} by subtracting the mean and then rescaling the coordinates of the points between $0$ and $1$ (between $0$ and half of the dimension for the general case).
    Hence, we obtain the following matrix whose rows are the coordinates of the points of the transformed
    point cloud $X$:
    $$ V=\begin{pmatrix}
0 & 0 \\
0.25 & 0.25 \\
0.75 & 0.75 \\
1 & 1 \\
0 & 1 \\
0.25 & 0.75 \\
0.75 & 0.25 \\
1 & 0 
\end{pmatrix}.  $$
\item Barycentric coordinates computation: we compute the barycentric coordinates of each point of $X$ with respect to the simplex $\sigma$.
To do so, we have to multiply the matrix $V$ by the matrix $M$. obtained using the formula from Lemma~\ref{lemma:barycentric_comp}:
\[\begin{pmatrix}
1& 0 & 0 \\
1& 0.25 & 0.25 \\
1& 0.75 & 0.75 \\
1& 1 & 1 \\
1& 0 & 1 \\
1& 0.25 & 0.75 \\
1& 0.75 & 0.25 \\
1& 1 & 0 
\end{pmatrix}\cdot  \begin{pmatrix}
1 & 0 & 0 \\
-\frac{1}{2} & \frac{1}{2} & 0 \\
-\frac{1}{2} & 0 & \frac{1}{2} 
\end{pmatrix}.
\]
\item SIMAP layer training: A first SIMAP layer can be trained using as input the matrix obtained in the previous step. After training, we obtained the following weight
matrix: 
$$
W_0=\begin{pmatrix}
0.46 & 0.48 \\
-0.15 & -0.34 \\
0.79 & 0.93 
\end{pmatrix}$$
\item Transfer learning:  to increase the capacity of the SIMAP layer, we  apply barycentric subdivisions and inherit the previous weight matrix.
Let us denote by $B_0$ the 
$25\times 7$ matrix whose entries are the barycentric coordinates of the vertices of $\Sd \sigma$ with respect to $\sigma$ in the correct row ordering, i.e. $B_0=\xi_{\Sd \sigma}(x)
\cdot \Omega^1$. Then, the init weight matrix 
$W_1$ of the new SIMAP layer is the
$25\times 2$ matrix:
$$W_1=B_0\cdot W_0\,.$$
\end{enumerate}

The steps 3 and 4 can be iterated as many times as needed.
In our case, we applied the methodology for two barycentric subdivisions and obtained the curves for the training accuracy and loss function shown in Figure~\ref{fig:example_loss_acc}.

In figure~\ref{fig:decision_boundary} we can see the decision boundaries for the different models with respect to $\sigma$, $\Sd \sigma$, and $\Sd^2 \sigma$. As expected, the complexity of the decision boundary increases when the barycentric subdivision process is
iterated. 
\end{example}

\section{Experiments}\label{sec:exper}
This section presents various experiments. In the first set of experiments, synthetic datasets were
used.
In the second set of experiments, the proposed methodology is applied as a final layer of a convolutional neural network to classify the MNIST dataset. In all experiments, we used the Adam \cite{DBLP:journals/corr/KingmaB14} training algorithm. 

\subsection{Experiments with synthetic datasets}\label{sec:synthetic_exp}

In this set of experiments, we used synthetic datasets composed of 500 points for binary classification from \cite{Guyon2003DesignOE} \texttt{scikit-learn} implementation 
for different numbers of features (2,3,4 and 5). The datasets were 
split into training and test sets with a proportion of $80\%$ and $20\%$, respectively. The steps followed are the same as in Example~\ref{example:XOR}. The SIMAP layers were trained for $1000$ epochs and barycentric subdivisions were applied two times. The results are shown in Table~\ref{table:results_synthetic} where the values are the mean of $5$ repetitions. As we can see from the results, when applying barycentric subdivisions, the SIMAP layer quickly reaches overfitting, being, in general, enough to apply just one barycentric subdivision for this dataset. 

\begin{table}[ht]
\caption{Loss and accuracy values on training and test set for the experiment in Section~\ref{sec:synthetic_exp}. }
\centering
\begin{tabular}{|c|c|cc|cc|}
\hline
\multirow{2}{*}{n} & \multirow{2}{*}{no. subdivisions} & \multicolumn{2}{c|}{Loss}         & \multicolumn{2}{c|}{Accuracy}     \\ \cline{3-6} 
                   &                                   & \multicolumn{1}{c|}{Train} & Test & \multicolumn{1}{c|}{Train} & Test \\ \hline
\multirow{3}{*}{2} & 0                                 & \multicolumn{1}{c|}{0.32}  & 0.33 & \multicolumn{1}{c|}{0.9}   & 0.87 \\ \cline{2-6} 
                   & 1                                 & \multicolumn{1}{c|}{0.26}  & 0.27 & \multicolumn{1}{c|}{0.91}  & 0.91 \\ \cline{2-6} 
                   & 2                                 & \multicolumn{1}{c|}{0.22}  & 0.26 & \multicolumn{1}{c|}{0.91}  & 0.89 \\ \hline
\multirow{3}{*}{3} & 0                                 & \multicolumn{1}{c|}{0.44}  & 0.4  & \multicolumn{1}{c|}{0.82}  & 0.84 \\ \cline{2-6} 
                   & 1                                 & \multicolumn{1}{c|}{0.39}  & 0.37 & \multicolumn{1}{c|}{0.83}  & 0.81 \\ \cline{2-6} 
                   & 2                                 & \multicolumn{1}{c|}{0.26}  & 0.48 & \multicolumn{1}{c|}{0.87}  & 0.83 \\ \hline
\multirow{3}{*}{4} & 0                                 & \multicolumn{1}{c|}{0.38}  & 0.36 & \multicolumn{1}{c|}{0.87}  & 0.87 \\ \cline{2-6} 
                   & 1                                 & \multicolumn{1}{c|}{0.29}  & 0.29 & \multicolumn{1}{c|}{0.88}  & 0.85 \\ \cline{2-6} 
                   & 2                                 & \multicolumn{1}{c|}{0.14}  & 1.04 & \multicolumn{1}{c|}{0.95}  & 0.83 \\ \hline
\multirow{3}{*}{5} & 0                                 & \multicolumn{1}{c|}{0.34}  & 0.31 & \multicolumn{1}{c|}{0.87}  & 0.87 \\ \cline{2-6} 
                   & 1                                 & \multicolumn{1}{c|}{0.24}  & 0.29 & \multicolumn{1}{c|}{0.92}  & 0.87 \\ \cline{2-6} 
                   & 2                                 & \multicolumn{1}{c|}{0.03}  & 0.99 & \multicolumn{1}{c|}{0.92}  & 0.76 \\ \hline
\end{tabular}
\label{table:results_synthetic}
\end{table}

\subsection{Convolutional network in conjunction with a SIMAP layer}

In this experiment, we  show
the case where a convolutional layer is combined with a SIMAP layer. 
We used the MNIST dataset.
\begin{enumerate}
    \item Dataset: The MNIST 
    dataset is composed of $60000$ training grayscale images and $10000$ test images of size $28\times 28$.
    \item Convolutional network training: A convolutional neural network is trained for $10$ epochs using the training set. 
    The architecture used is
       the following:
    \begin{align*}
&\text{Conv2D} & (26, 26, 28)  \\
&\text{MaxPooling2D} & (13, 13, 28)  \\
&\text{Conv2D} & (1, 11, 64)  \\
&\text{MaxPooling2D} & (5, 5, 64) \\
&\text{Conv2D} & (1, 1, 4)  \\
&\text{Flatten} & (4)  \\
&\text{Dense} & (64)  \\
&\text{Dense} & (10) \\
\end{align*}
This architecture reached an accuracy of $0.98$ and a loss value of $0.057$ on the test set.
\item Dataset preparation: The output of the {\it flatten} layer is used as a four-dimensional point cloud. Next, it is scaled and translated within the simplex $\sigma$ defined 
in Lemma~\ref{lemma_tetra}.
We denote by $X$ the transformed point cloud.
\item Computation of the barycentric coordinates: The barycentric coordinates of the points of $X$ with respect to the simplex $\sigma$, as well as their barycentric coordinates with respect to the first and second barycentric subdivisions ($\Sd \sigma$, and $\Sd^2 \sigma$), are calculated following the methodology presented in Lemma~\ref{lemma:barycentric_comp}.
\item SIMAP layer training: Sequentially, three SIMAP layers were trained, with weights transferred from \ro{a}
SIMAP layer to the next as explained in Section~\ref{subsec:training}. The accuracy and loss values reached in the test set are shown in Table~\ref{tab:model_performance}.
\end{enumerate}

\begin{table}[ht]
\centering
\caption{Loss and accuracy values of the SIMAP layer with the different 
barycentric subdivision iterations and the Convolutional Neural Network without the SIMAP layer.}
\begin{tabular}{|c|c|c|c|}
\hline
                              & Number of subdivisions                          & Loss & Accuracy \\ \cline{2-4} 
                              & 0                                               & 0.39 & 0.97     \\ \cline{2-4} 
                              & 1                                               & 0.06 & 0.98     \\ \cline{2-4} 
\multirow{-4}{*}{SIMAP layer} & 2                                               & 0.05 & 0.98     \\ \hline
CNN                           & \cellcolor[HTML]{C0C0C0}{\color[HTML]{9B9B9B} } & 0.06 & 0.98     \\ \hline
\end{tabular}
\label{tab:model_performance}
\end{table}

Based on the results obtained, we can conclude that achieving good performance does not require many barycentric subdivisions. Furthermore, the performance of convolutional neural networks is not adversely affected by the application of the SIMAP
layer.

\section{Conclusions and future work}
Blind optimization methods have shown to be an efficient method for finding an accurate set of weights to solve real-world problems with neural networks. This ability to solve problems anyhow has been the main target during the first years of the development of Deep Learning. Nevertheless, the wide use of Artificial Intelligence (AI) methods in many domestic areas is turning the perception of the social use of AI. The development of AI methods that solve problems anyhow is no longer socially accepted. In this way, in the next years, the interpretability of the model, i.e., the ability to interpret the decision-making protocols in AI in a human-readable way, will be an inexcusable requirement in the development of AI. This paper is a contribution to this research line, and, to the best of our knowledge, no other layer based on simplicial map methods for improving the interpretability of neural networks has been presented to establish a comparison.

Specifically, the main contributions of this work are the SIMAP layer definition and
an efficient matrix-based algorithm for their implementation. The benefits of this new layer are their interpretability
and their possibility to be applied with deep learning models such as convolutional neural networks. In the future, we want to exploit their explainability abilities, as well as develop specific barycentric subdivisions to reduce the number of needed simplices and, therefore, to reduce the complexity.

\section*{Code availability}

The code for the experiments is available in \url{https://github.com/Cimagroup/SIMAP-layer}.

\section*{Acknowledgments}

The work was supported in part by the European Union HORIZON-CL4-2021-HUMAN-01-01 under grant agreement 101070028 (REXASI-PRO) and by  
TED2021-129438B-I00 / AEI/10.13039/501100011033 / Unión Europea NextGenerationEU/PRTR.


 
\bibliography{biblio}
\bibliographystyle{IEEEtran}
%


\newpage

 




\vfill

\end{document}